\documentclass{article}
\usepackage{amsmath,amsfonts,bm, amssymb, amsthm}
\usepackage{multirow}
\usepackage{mathtools}
\usepackage{threeparttable}
\usepackage{tablefootnote}
\usepackage[table]{xcolor}
\usepackage{xspace}
\usepackage{microtype}
\usepackage{graphicx}
\usepackage{subfigure}
\usepackage{booktabs} 

\usepackage{hyperref}
\usepackage{cleveref}

\newtheorem{proposition}{Proposition}

\usepackage[accepted]{icml2021}

\icmltitlerunning{Learning Gradient Fields for Molecular Conformation Generation}
\newcommand{\method}{ConfGF\xspace}
\newcommand{\methodd}{ConfGFDist\xspace}

\makeatletter
\DeclareRobustCommand\onedot{\futurelet\@let@token\@onedot}
\def\@onedot{\ifx\@let@token.\else.\null\fi\xspace}

\def\ie{\emph{i.e}\onedot}

\makeatother

\begin{document}

\twocolumn[
\icmltitle{Learning Gradient Fields for Molecular Conformation Generation}



\icmlsetsymbol{equal}{*}

\begin{icmlauthorlist}
\icmlauthor{Chence Shi}{equal,mila,udem}
\icmlauthor{Shitong Luo}{equal,pku}
\icmlauthor{Minkai Xu}{mila,udem}
\icmlauthor{Jian Tang}{mila,cifar,hec}

\end{icmlauthorlist}

\icmlaffiliation{mila}{Mila - Quebec AI Institute, Montr\'eal, Canada}
\icmlaffiliation{udem}{University of Montr\'eal, Montr\'eal, Canada}
\icmlaffiliation{hec}{HEC Montr\'eal, Montr\'eal, Canada}
\icmlaffiliation{cifar}{CIFAR AI Research Chair}

\icmlaffiliation{pku}{Peking University}

\icmlcorrespondingauthor{Jian Tang}{jian.tang@hec.ca}

\icmlkeywords{Machine Learning, ICML}

\vskip 0.3in
]



\printAffiliationsAndNotice{\icmlEqualContribution} 

\begin{abstract}

We study a fundamental problem in computational chemistry known as molecular conformation generation, trying to predict stable 3D structures from 2D molecular graphs. 
Existing machine learning approaches usually first predict distances between atoms and then generate a 3D structure satisfying the distances, where noise in predicted distances may induce extra errors during 3D coordinate generation. Inspired by the traditional force field methods for molecular dynamics simulation, in this paper, we propose a novel approach called~\method by directly estimating the gradient fields of the log density of atomic coordinates. The estimated gradient fields allow directly generating stable conformations via Langevin dynamics. However, the problem is very challenging as the gradient fields are roto-translation equivariant.
We notice that estimating the gradient fields of atomic coordinates can be translated to estimating the gradient fields of interatomic distances, and hence develop a novel algorithm based on recent score-based generative models to effectively estimate these gradients. Experimental results across multiple tasks show that \method outperforms previous state-of-the-art baselines by a significant margin.
The code is available at \url{https://github.com/DeepGraphLearning/ConfGF}.

\end{abstract}

\section{Introduction}
\label{sec:intro}

Graph-based representations of molecules, where nodes represent atoms and edges represent bonds, have been widely applied to a variety of molecular modeling tasks, ranging from property prediction~\cite{gilmer2017neural, Hu*2020Strategies} to molecule generation~\cite{jin2018junction, you2018graph, shi2020graphaf, xie2021mars}.
However, in reality a more natural and intrinsic representation for molecules is their three-dimensional structures, where atoms are represented by 3D coordinates. These structures are also widely known as molecular geometry or \textit{conformation}.


Nevertheless, obtaining valid and stable conformations for a molecule remains a long-standing challenge.
Existing approaches to molecular conformation generation mainly rely on molecular dynamics (MD)~\cite{de2016role}, where coordinates of atoms are sequentially updated based on the forces acting on each atom. 
The per-atom forces come either from computationally intensive Density Functional Theory (DFT)~\cite{parr1989DFT} or from hand-designed force fields~\cite{rappe1992uff, halgren1996merck-extension} which are often crude approximations of the actual molecular
energies~\cite{Kanal2017ASA}.
Recently, there are growing interests in developing machine learning methods for conformation generation~\cite{mansimov19molecular, simm2020GraphDG, xu2021cgcf}. In general, these approaches ~\cite{simm2020GraphDG, xu2021cgcf} are divided into two stages, which first predict the molecular distance geometry, i.e., interatomic distances, and then convert the distances to 3D coordinates via post-processing algorithms~\cite{liberti2014euclidean}. 
Such methods effectively model the roto-translation equivariance of molecular conformations and have achieved impressive performance. Nevertheless, such approaches generate conformations in a two-stage fashion, where noise in generated distances may affect 3D coordinate reconstruction, often leading to less accurate or even erroneous structures. Therefore, we are seeking for an approach that is able to generate molecular conformations within a single stage.


Inspired by the traditional force field methods for molecular dynamics simulation~\cite{Frenkel1996UnderstandingMS}, in this paper, we propose a novel and principled approach called \method for molecular conformation generation.
The basic idea is to directly learn the gradient fields of the log density w.r.t. the atomic coordinates, which can be viewed as pseudo-forces acting on each atom.
With the estimated gradient fields, conformations can be generated directly using Langevin dynamics~\cite{Max2011langevin}, which amounts to moving atoms toward the high-likelihood regions guided by pseudo-forces. The key challenge with this approach is that such gradient fields of atomic coordinates are roto-translation equivariant, i.e., the gradients rotate together with the molecular system and are invariant under translation.

To tackle this challenge, we observe that the interatomic distances are continuously differentiable w.r.t. atomic coordinates. Therefore, we propose to first estimate the gradient fields of the log density w.r.t. interatomic distances, and show that under a minor assumption, the gradient fields of the log density of atomic coordinates can be calculated from these w.r.t. distances in closed form using chain rule. We also demonstrate that such designed gradient fields satisfy roto-translation equivariance. In contrast to distance-based methods~\cite{simm2020GraphDG, xu2021cgcf}, our approach generates conformations in an one-stage fashion, i.e., the sampling process is done directly in the coordinate space, which avoids the errors induced by post-processing and greatly enhances generation quality. 
In addition, our approach
requires no surrogate losses and allows flexible network architectures, making it more capable to generate diverse and realistic conformations.

We conduct extensive experiments on GEOM~\cite{axelrod2020geom} and ISO17~\cite{simm2020GraphDG} benchmarks, and compare \method against previous state-of-the-art neural-methods as well as the empirical method on multiple tasks
ranging from conformation generation and distance modeling to property prediction.
Numerical results show that \method outperforms previous state-of-the-art baselines by a clear margin.

\section{Related Work}
\label{sec:related}

\textbf{Molecular Conformation Generation}
Existing works on molecular conformation generation mainly rely on molecular dynamics (MD)~\cite{de2016role}.
Starting from an initial conformation and a physical model for interatomic potentials~\cite{parr1989DFT, rappe1992uff}, the conformation is sequentially updated based on the forces acting on each atom.
However, these methods are inefficient, especially when molecules are large as they involve computationally expensive quantum mechanical calculations~\cite{parr1989DFT, shim2011computational, ballard2015exploiting}.
Another much more efficient but less accurate sort of approach is the empirical method, which fixes distances and angles between atoms in a molecule to idealized values according to a set of rules~\cite{Blaney2007DistanceGI}.

Recent advances in deep generative models open the door to data-driven approaches, which have shown great potential to strike a good balance between computational efficiency and accuracy.
For example, \citet{mansimov19molecular} introduce the Conditional Variational Graph Auto-Encoder (CVGAE) which takes molecular graphs as input and directly generates 3D coordinates for each atom. This method does not manage to model the roto-translation equivariance of molecular conformations, resulting in great differences between generated structures and ground-truth structures.
To address the issue, \citet{simm2020GraphDG} and \citet{xu2021cgcf} propose to model interatomic distances of molecular conformations using VAE and Continuous Flow respectively.
These approaches preserve roto-translation equivariance of molecular conformations and consist of two stages --- first generating distances between atoms and then feeding the distances to the distance geometry algorithms~\cite{liberti2014euclidean} to reconstruct 3D coordinates.
However, the distance geometry algorithms are vulnerable to noise in generated distances, which may lead to inaccurate or even erroneous structures.
Another line of research~\cite{Gogineni2020TorsionNet} leverages reinforcement learning for conformation generation by sequentially determining torsional angles, which relies on an additional classical force field for state transition and reward evaluation. 
Nevertheless, it is incapable of modeling other geometric properties such as bond angles and bond lengths, which distinguishes it from all other works.


\textbf{Neural Force Fields}
Neural networks have also been employed to estimate molecular potential energies and force fields~\cite{schutt2017schnet, zhang2018deepmd, hu2021forcenet}. 
The goal of these models is to predict energies and forces as accurate as possible, serving as an alternative to quantum chemical methods such as Density Function Theory (DFT) method.
Training these models requires access to molecular dynamics trajectories along with ground-truth energies and forces from expensive quantum mechanical calculations.
Different from these methods, our goal is to \textit{generate equilibrium conformations within a single stage}.
To this end, we define gradient fields\footnote{Such gradient fields, however, are not force fields as they are not energy conserving.} analogous to force fields, which serve as pseudo-forces acting on each atom that gradually move atoms toward the high-density areas until they converges to an equilibrium.
The only data we need to train the model is equilibrium conformations.



    

\section{Preliminaries}

\subsection{Problem Definition}
\textbf{Notations}
In this paper, a molecular graph is represented as an undirected graph $G=(V, E)$, where $V = \{v_1, v_2, \cdots, v_{\vert V \vert}\}$ is the set of nodes representing atoms, and $E = \{e_{ij} \mid (i, j) \subseteq V \times V\}$ is the set of edges between atoms in the molecule. 
Each node $v_i \in V$ is associated with a nuclear charge $Z_i$ and a 3D vector $\mathbf{r}_i \in \mathbb{R}^3$ , indicating its atomic type and atomic coordinate respectively.
Each edge $e_{ij} \in E$ is associated with a bond type and a scalar $d_{ij} = \Vert \mathbf{r}_i - \mathbf{r}_j \Vert_2$ denoting the Euclidean distance between the positions of $v_i$ and $v_j$. All distances between connected nodes can be represented as a vector $\mathbf{d} = \{d_{ij}\} \in \mathbb{R}^{\vert E \vert}$.

As the bonded edges in a molecule are not sufficient to characterize a conformation (bonds are rotatable), we extend the original molecular graph by adding auxiliary edges, i.e., virtual bonds, between atoms that are 2 or 3 hops away from each other, which is a widely-used technique in previous work~\cite{simm2020GraphDG, xu2021cgcf} for reducing the degrees of freedom in 3D coordinates.
The extra edges between second neighbors help fix angles between atoms, while those between third neighbors fix dihedral angles. 
Hereafter, we assume all molecular graphs are extended unless stated.

\textbf{Problem Definition}
Given an extended molecular graph $G=(V, E)$, our goal is to learn a generative model that generates molecular conformations $\mathbf{R}=(\mathbf{r}_1, \mathbf{r}_2, \cdots, \mathbf{r}_{\vert  V \vert}) \in \mathbb{R}^{\vert V \vert \times 3}$ based on the molecular graph.

\subsection{Score-Based Generative Modeling}
The \textit{score function} $
\mathbf{s}(\mathbf{x})$ of a continuously differentiable probability density $p(\mathbf{x})$ is defined as $\nabla_{\mathbf{x}} \log p(\mathbf{x})$.
Score-based generative modeling~\cite{song2019score, song2020slicedscore, meng2020autoregressive, song2020scoretech} is a class of generative models that approximate the score function of $p(\mathbf{x})$ using neural networks and generate new samples with Langevin dynamics~\cite{Max2011langevin}. Modeling the score function instead of the density function allows getting rid of the intractable partition function of $p(\mathbf{x})$ and avoiding extra computation on getting the higher-order gradients of an energy-based model~\cite{song2019score}.

To address the difficulty of estimating score functions in the regions without training data, \citet{song2019score} propose to perturb the data using Gaussian noise of various intensities and jointly estimate the score functions, i.e., $\mathbf{s}_\theta (\mathbf{x};\sigma_i)$, for all noise levels.
In specific, given a data distribution $p(\mathbf{x})$ and a sequence of noise levels $\{\sigma_i\}_{i=1}^L$ with a noise distribution $q_{\sigma_i}(\tilde{\mathbf{x}} \mid \mathbf{x})$, e.g., $\mathcal{N}(\tilde{\mathbf{x}} \mid \mathbf{x}, \sigma_i^2 \mathbf{I})$, the training objective $\ell(\theta; \sigma_i)$ for each noise level $\sigma_i$ is as follows:
\begin{equation}
    \frac{1}{2}
    \mathbb{E}_{p(\mathbf{x})}
    \mathbb{E}_{q_{\sigma_i}(\tilde{\mathbf{x}} \mid \mathbf{x})} 
    \big[\big\Vert
    \mathbf{s}_\theta(\tilde{\mathbf{x}}, \sigma_i) 
    -\nabla_{\tilde{\mathbf{x}}} \log q_{\sigma_i}(\tilde{\mathbf{x}} \mid \mathbf{x})
    \big\Vert_2^2 \big].
\end{equation}

After the noise conditional score networks $\mathbf{s}_\theta(\mathbf{x}, \sigma_i)$ are trained, samples are generated using annealed Langevin dynamics~\cite{song2019score}, where samples from each noise level serve as initializations for Langevin dynamics of the next noise level. For detailed design choices of noise levels and sampling hyperparameters, we refer readers to~\citet{song2020scoretech}.

\subsection{Equivariance in Molecular Geometry}
\label{subsec: equivariance}

\textit{Equivariance} is ubiquitous in physical systems, e.g., 3D roto-translation equivariance of molecular conformations or point clouds~\cite{Thomas2018TensorFN, Weiler20183DSC, sGDML, fuchs2020se3, Miller2020RelevanceOR, simm2021symmetryaware}. 
Endowing model with such inductive biases is critical for better generalization and successful learning~\cite{kohler20eqflow}.
For molecular geometry modeling, such equivariance is guaranteed by either taking gradients of a predicted invariant scalar energy~\cite{Klicpera2020Directional, schutt2017schnet} or using equivariant networks~\cite{fuchs2020se3, satorras2021en}.
Formally, a function $\mathcal{F}: \mathcal{X} \rightarrow \mathcal{Y}$ 
being equivariant can be represented as follows:\footnote{Strictly speaking, $\rho$ does not have to be the same on both sides of the equation, as long as it is a representation of the same group. Note that \textit{invariance} is a special case of equivariance.}
\begin{equation}
\label{eq:equivariance}
    \mathcal{F} \circ \rho (x) = \rho \circ \mathcal{F} (x),
\end{equation}
where $\rho$ is a transformation function, e.g., rotation. Intuitively, Eq.~\ref{eq:equivariance} says that applying the $\rho$ on the input has the same effect as applying it to the output.

In this paper, we aim to model the score function of $p(\mathbf{R} \mid G)$, i.e., $\mathbf{s}(\mathbf{R}) =\nabla_{\mathbf{R}}\log p(\mathbf{R} \mid G)$. 
As $\log p(\mathbf{R} \mid G)$ is roto-translation invariant with respect to conformations, the score function $\mathbf{s}(\mathbf{R})$ is therefore roto-translation equivariant with respect to conformations. To pursue a generalizable and elegant approach, we explicitly build such equivariance into the model design.

\begin{figure*}[t]
	\centering
    \includegraphics[width=0.96\linewidth]{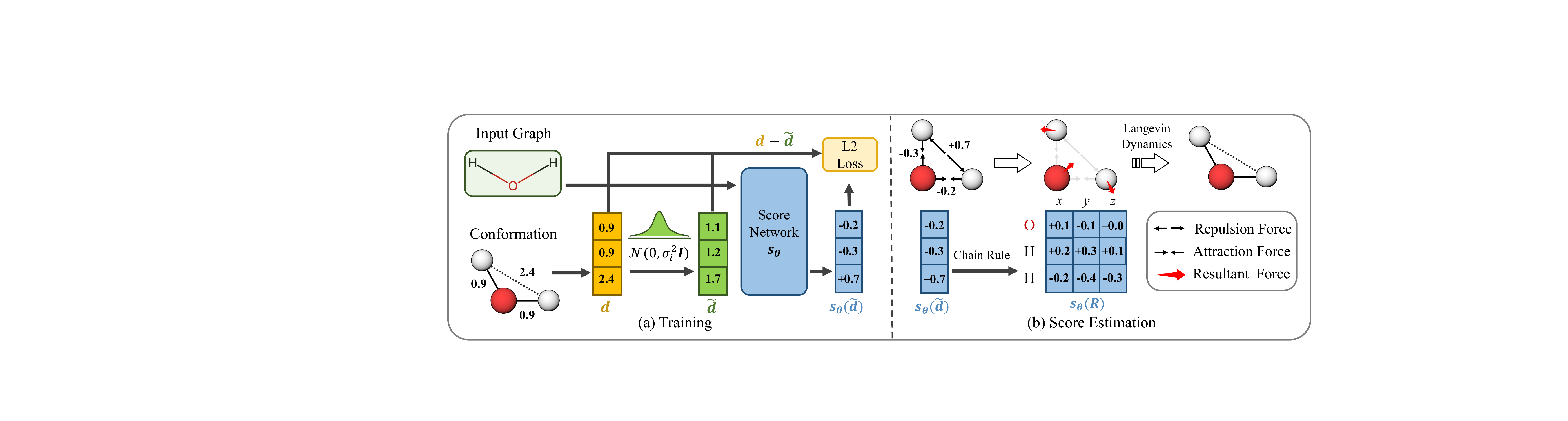}
    \caption{Overview of the proposed \method approach. (a) Illustration of the training procedure. We perturb the interatomic distances with random Gaussian noise of various magnitudes, and train the noise conditional score network with denoising score matching.
    (b) Score estimation of atomic coordinates via chain rule. The procedure amounts to calculating the resultant force from a set of interatomic forces.
    }
    \label{fig:training}
    \vspace{-2pt}
\end{figure*}

\section{Proposed Method}
\label{sec:method}

\subsection{Overview}
Our approach extends the denoising score matching~\cite{vincent2011denoising, song2019score} to molecular conformation generation by estimating the gradient fields of the log density of atomic coordinates, i.e., $\nabla_{\mathbf{R}}\log p(\mathbf{R} \mid G)$. 
Directly parameterizing such gradient fields using vanilla Graph Neural Networks (GNNs)~\cite{duvenaud2015convolutional,kearnes2016molecular,kipf2016semi,gilmer2017neural,schutt2017quantum, schutt2017schnet} is problematic, as these GNNs only leverage the node, edge or distance information, resulting in node representations being roto-translation invariant. 
However, the gradient fields of atomic coordinates should be roto-translation equivariant (see Section~\ref{subsec: equivariance}), i.e., the vector fields should rotate together with the molecular system and be invariant under translation.

To tackle this issue, we assume that the log density of atomic coordinates $\mathbf{R}$ given the molecular graph $G$ can be parameterized as $\log p_\theta(\mathbf{R} \mid G) \triangleq f_G \circ g_G(\mathbf{R}) = f_G(\mathbf{d})$ up to a constant, where $g_G: \mathbb{R}^{\vert V\vert \times 3} \rightarrow \mathbb{R}^{\vert E \vert}$ denotes a function that maps a set of atomic coordinates to a set of interatomic distances, and
$f_G: \mathbb{R}^{\vert E \vert} \rightarrow \mathbb{R}$ 
is a graph neural network that estimates the negative energy of a molecule based on the interatomic distances $\mathbf{d}$ and the graph representation $G$. 
Such design is a common practice in existing literature~\cite{schutt2017schnet, Hu*2020Strategies, Klicpera2020Directional}, which is favored as it ensures several physical invariances, e.g., rotation and translation invariance of energy prediction. 
We observe that the interatomic distances $\mathbf{d}$ are continuously differentiable w.r.t. atomic coordinates $\mathbf{R}$. 
Therefore, the two gradient fields, $\nabla_{\mathbf{R}}\log p_\theta(\mathbf{R} \mid G)$ and $\nabla_{\mathbf{d}}\log p_\theta(\mathbf{d} \mid G)$, are connected by chain rule, and the gradients can be backpropagated from $\mathbf{d}$ to $\mathbf{R}$ as follows:

\begin{equation}
\label{eq: chainrule}
\begin{aligned}
    \forall i, \mathbf{s}_\theta(\mathbf{R})_i
    &= \frac{\partial f_G(\mathbf{d})}{\partial \mathbf{r}_i} 
    =\sum\limits_{(i,j), e_{ij} \in E} \frac{\partial f_G(\mathbf{d})}{\partial d_{ij}} \cdot \frac{\partial d_{ij}}{\partial \mathbf{r}_i} \\
    &= \sum\limits_{j \in N(i)}
    \frac{1}{d_{ij}} \cdot \frac{\partial f_G(\mathbf{d})}{\partial d_{ij}} \cdot (\mathbf{r}_i - \mathbf{r}_j) \\
    &= \sum\limits_{j \in N(i)}
    \frac{1}{d_{ij}} \cdot \mathbf{s}_\theta(\mathbf{d})_{ij} \cdot (\mathbf{r}_i - \mathbf{r}_j),
\end{aligned}
\end{equation}
where $\mathbf{s}_\theta(\mathbf{R})_i=\nabla_{\mathbf{r}_i}\log p_\theta(\mathbf{R} \mid G)$ and $\mathbf{s}_\theta(\mathbf{d})_{ij}=\nabla_{d_{ij}}\log p_\theta(\mathbf{d} \mid G)$. Here $N(i)$ denotes $i^{th}$ atom{'}s neighbors. The second equation of Eq.~\ref{eq: chainrule} holds because of the chain rule over the partial derivative of the composite function $f_G(\mathbf{d}) = f_G \circ g_G(\mathbf{R})$.

Inspired by the above observations, we take the interatomic pairwise distances as intermediate variables and propose to first estimate the gradient fields of interatomic distances corresponding to different noise levels, i.e., training a noise conditional score network $\mathbf{s}_\theta(\mathbf{d}, \sigma)$ to approximate $\nabla_{\tilde{\mathbf{d}}}\log q_{\sigma}(\tilde{\mathbf{d}} \mid G)$. 
We then calculate the gradient fields of atomic coordinates, i.e., $\mathbf{s}_\theta(\mathbf{R}, \sigma)$, from $\mathbf{s}_\theta(\mathbf{d}, \sigma)$ via Eq.~\ref{eq: chainrule}, without relying on any extra parameters. 
Such designed score network $\mathbf{s}_\theta(\mathbf{R}, \sigma)$ enjoys the property of roto-translation equivariance. Formally, we have:
\begin{proposition}
\label{prop:roto}
\textit{Under the assumption that we parameterize $\log p_\theta(\mathbf{R} \mid G)$ as a function of interatomic distances $\mathbf{d}$ and molecular graph representation $G$, the score network defined in Eq.~\ref{eq: chainrule} is roto-translation equivariant, i.e., the gradients rotate together with the molecule system and are invariant under translation.
}
\end{proposition}
\begin{proof}[Proof Sketch]
The score network of distances $\mathbf{s}_\theta(\mathbf{d})$ is roto-translation invariant as it only depends on interatomic distances, and the vector
$(\mathbf{r}_i - \mathbf{r}_j) / d_{ij}$ rotates together with the molecule system, which endows the score network of atomic coordinates $\mathbf{s}_\theta(\mathbf{R})$ with roto-translation equivariance.
See Supplementary material for a formal proof.
\end{proof}


After training the score network, based on the estimated gradient fields of atomic coordinates corresponding to different noise levels, conformations are generated by annealed Langevin dynamics~\cite{song2019score}, combining information from all noise levels. 
The overview of \method is illustrated in Figure~\ref{fig:training} and Figure~\ref{fig:sampling}.
Below we describe the design of the noise conditional score network in Section~\ref{subsec:scorenet}, introduce the generation procedure in Section~\ref{subsec:generation}, and show how the designed gradient fields are connected with the classical laws of mechanics in Section~\ref{subsec: discussion}.

\begin{figure*}[t]
	\centering
    \includegraphics[width=0.97\linewidth]{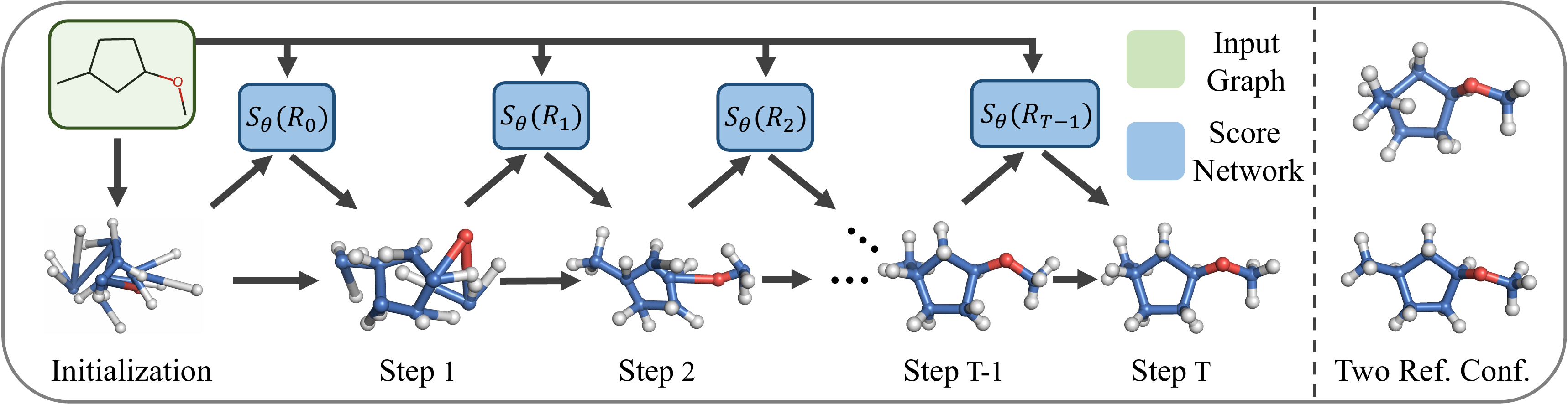}
    \caption{Generation procedure of the proposed \method via Langevin dynamics. Starting from a random initialization, the conformation is sequentially updated with the gradient information of atomic coordinates calculated from the score network.
    }
    \label{fig:sampling}
    \vspace{+2pt}
\end{figure*}

\subsection{Noise Conditional Score Networks for Distances}
\label{subsec:scorenet}

Let $\{\sigma_i\}_{i=1}^L$ be a positive geometric progression with common ratio $\gamma$, i.e., ${\sigma_{i}} / {\sigma_{i-1}} = \gamma$. We define the perturbed conditional distance distribution $q_{\sigma}(\tilde{\mathbf{d}} \mid G)$ to be $\int p(\mathbf{d} \mid G) \mathcal{N}(\tilde{\mathbf{d}} \mid \mathbf{d}, \sigma_i^2 \mathbf{I}) d\mathbf{d} $. In this setting, we aim to learn a conditional score network to jointly estimate the scores of all perturbed distance distributions, i.e., $\forall \sigma \in \{\sigma_i\}_{i=1}^L: \mathbf{s}_\theta(\tilde{\mathbf{d}}, \sigma) \approx \nabla_{\tilde{\mathbf{d}}}\log q_{\sigma}(\tilde{\mathbf{d}} \mid G)$. 
Note that $\mathbf{s}_\theta(\tilde{\mathbf{d}}, \sigma) \in \mathbb{R}^{\vert E \vert}$, we therefore formulate it as an edge regression problem.

Given a molecular graph $G = (V, E)$ with a set of interatomic distances $\mathbf{d} \in \mathbb{R}^{\vert E \vert}$ computed from its conformation $\mathbf{R} \in \mathbb{R}^{\vert V \vert \times 3}$, we first embed the node attributes and the edge attributes into low dimensional feature spaces using feedforward networks:
\begin{equation}
\label{eq:embed}
\begin{aligned}
    \mathbf{h}_i^0 &= \operatorname{MLP}(Z_i), \quad \forall v_i \in V, \\
    \mathbf{h}_{e_{ij}} &= \operatorname{MLP}(e_{ij}, d_{ij}), \quad \forall e_{ij} \in E.
\end{aligned}
\end{equation}
We then adopt a variant of Graph Isomorphism Network (GINs)~\cite{xu2018how, Hu*2020Strategies} to compute node embeddings based on graph structures. At each layer, node embeddings are updated by aggregating messages from neighboring nodes and edges:
\begin{equation}
\label{eq:gnn}
\begin{aligned}
    \mathbf{h}_i^l = \operatorname{MLP}(\mathbf{h}_i^{l-1} + \sum\limits_{j \in N(i)}\operatorname{ReLU}(\mathbf{h}_j^{l-1} + \mathbf{h}_{e_{ij}})),
\end{aligned}
\end{equation}
where $N(i)$ denotes $i^{th}$ atom{'}s neighbors. After $N$ rounds of message passing, we compute the final edge embeddings by concatenating the corresponding node embeddings for each edge as follows:
\begin{equation}
\begin{aligned}
    \mathbf{h}_{e_{ij}}^{o} = \mathbf{h}_i^{N} \ \Vert \ \mathbf{h}_j^{N} \ \Vert \  \mathbf{h}_{e_{ij}},
\end{aligned}
\end{equation}
where $\Vert$ denotes the vector concatenation operation, and $\mathbf{h}_{e_{ij}}^{o}$ denotes the final embeddings of edge $e_{ij} \in E$.
Following the suggestions of~\citet{song2020scoretech}, we parameterize the noise conditional network with $\mathbf{s}_\theta(\tilde{\mathbf{d}}, \sigma) = \mathbf{s}_\theta(\tilde{\mathbf{d}}) / \sigma$ as follows:
\begin{equation}
\begin{aligned}
    \mathbf{s}_\theta(\tilde{\mathbf{d}})_{ij} &= \operatorname{MLP}(\mathbf{h}_{e_{ij}}^{o}), \quad \forall e_{ij} \in E,
\end{aligned}
\end{equation}
where $\mathbf{s}_\theta(\tilde{\mathbf{d}})$ is an unconditional score network, and the $\operatorname{MLP}$ maps edge embeddings to a scalar for each $e_{ij} \in E$. Note that $\mathbf{s}_\theta(\tilde{\mathbf{d}}, \sigma)$ is roto-translation invariant as we only use interatomic distances.
Since $\nabla_{\tilde{\mathbf{d}}}\log q_{\sigma}(\tilde{\mathbf{d}} \mid \mathbf{d}, G) = -(\tilde{\mathbf{d}} - \mathbf{d}) / \sigma^2$, the training loss of $\mathbf{s}_\theta(\tilde{\mathbf{d}}, \sigma)$ is thus:

\begin{equation}
\begin{aligned}
\frac{1}{2L}\sum\limits_{i=1}^L  \lambda(\sigma_i) \mathbb{E}_{p(\mathbf{d} \mid G)}\mathbb{E}_{q_{\sigma_i}(\tilde{\mathbf{d}} \mid \mathbf{d}, G)} \left[\left\Vert
    \frac{\mathbf{s}_\theta(\tilde{\mathbf{d}})}{\sigma_i} + \frac{\tilde{\mathbf{d}} - \mathbf{d}}{\sigma_i^2}
    \right\Vert_2^2 \right],
\end{aligned}    
\end{equation}
where $\lambda(\sigma_i) = \sigma_i^2$ is the coefficient function weighting losses of different noise levels according to~\citet{song2019score}, and
all expectations can be efficiently estimated using Monte Carlo estimation.

\subsection{Conformation Generation via Annealed Langevin Dynamics}
\label{subsec:generation}
Given the molecular graph representation $G$ and the well-trained noise conditional score networks, molecular conformations are generated using annealed Langevin dynamics, i.e., gradually anneal down the noise level from large to small. 
We start annealed Langevin dynamics by first sampling an initial conformation $\mathbf{R}_0$ from some fixed prior distribution, e.g., uniform distribution or Gaussian distribution.
Empirically, the choice of the prior distribution is arbitrary as long as the supports of the perturbed distance distributions cover the prior distribution, and we take the Gaussian distribution as prior distribution in our case.
Then, we update the conformations by sampling from a series of trained noise conditional score networks $\{\mathbf{s}_\theta(\mathbf{R}, \sigma_i)\}_{i=1}^L$ sequentially.
For each noise level $\sigma_i$, starting from the final sample from the last noise level,
we run Langevin dynamics for $T$ steps with a gradually decreasing step size $\alpha_i = \epsilon \cdot \sigma_i^2/\sigma_L^2$ for each noise level.
In specific, at each sampling step $t$, we first obtain the interatomic distances $\mathbf{d}_{t-1}$ based on the current conformation $\mathbf{R}_{t-1}$, and calculate the score function $\mathbf{s}_\theta(\mathbf{R}_{t-1}, \sigma_i)$ via Eq.~\ref{eq: chainrule}. The conformation is then updated using the gradient information from the score network. The whole sampling algorithm is presented in Algorithm~\ref{alg:anneal}.

\begin{algorithm}[H]
	\caption{Annealed Langevin dynamics sampling}
	\label{alg:anneal}
    \renewcommand\algorithmiccomment[1]{\hfill $\triangleright$ {#1}}	
	\begin{algorithmic}[1]
	    \INPUT{molecular graph $G$, noise levels $\{\sigma_i\}_{i=1}^L$, the smallest step size $\epsilon$, and the number of sampling steps per noise level $T$.}
	    \STATE{Initialize conformation $\mathbf{R}_0$ from a prior distribution}
	    \FOR{$i \gets 1$ to $L$}
	        \STATE{$\alpha_i \gets \epsilon \cdot \sigma_i^2/\sigma_L^2$} \COMMENT{$\alpha_i$ is the step size.}
            \FOR{$t \gets 1$ to $T$}
                \STATE{$\mathbf{d}_{t-1} = g_G(\mathbf{R}_{t-1})$}
                \COMMENT{get distances from $\mathbf{R}_{t-1}$}
                \STATE{$\mathbf{s}_\theta(\mathbf{R}_{t-1}, \sigma_i) \gets \operatorname{convert}(\mathbf{s}_\theta(\mathbf{d}_{t-1}, \sigma_i))$}
                \COMMENT{Eq.~\ref{eq: chainrule}.}
                \STATE{Draw $\mathbf{z}_t \sim \mathcal{N}(0, \mathbf{I})$}
                \STATE{$\mathbf{R}_t \gets \mathbf{R}_{t-1} + \alpha_i \mathbf{s}_{\theta}(\mathbf{R}_{t-1}, \sigma_i) + \sqrt{2 \alpha_i}\mathbf{z}_t$}
            \ENDFOR
            \STATE{$\mathbf{R}_0 \gets \mathbf{R}_T$}
        \ENDFOR
        \OUTPUT{Generated conformation $\mathbf{R}_T$.}
	\end{algorithmic}
\end{algorithm}

\subsection{Discussion}
\label{subsec: discussion}
\textbf{Physical Interpretation of Eq.~\ref{eq: chainrule}}
From a physical perspective, $\mathbf{s}_\theta(\mathbf{d}, \sigma)$ reflects how the distances between connected atoms should change to increase the probability density of the conformation $\mathbf{R}$, which can be viewed as a network that predicts the repulsion or attraction forces between atoms. 
Moreover, as shown in Figure~\ref{fig:training}, $\mathbf{s}_\theta(\mathbf{R}, \sigma)$ can be viewed as resultant forces acting on each atom computed from Eq.~\ref{eq: chainrule}, where $(\mathbf{r}_i- \mathbf{r}_j) / d_{ij}$ is a unit vector indicating the direction of each component force and $\mathbf{s}_\theta(\mathbf{d}, \sigma)_{ij}$ specifies the strength of each component force. 

\textbf{Handling Stereochemistry}
Given a molecular graph, our \method generates all possible stereoisomers due to the stochasticity of the Langevin dynamics. 
Nevertheless, our model is very general and can be extended to handle stereochemistry by predicting the gradients of torsional angles with respect to coordinates, as torsional angles can be calculated in closed-form from 3D coordinates and are also invariant under rotation and translation. Since the stereochemistry study of conformation generation is beyond the scope of this work, we leave it for future work.

\section{Experiments}
\label{sec:experiment}

\begin{figure*}[htbp]
	\centering
    \includegraphics[width=\linewidth]{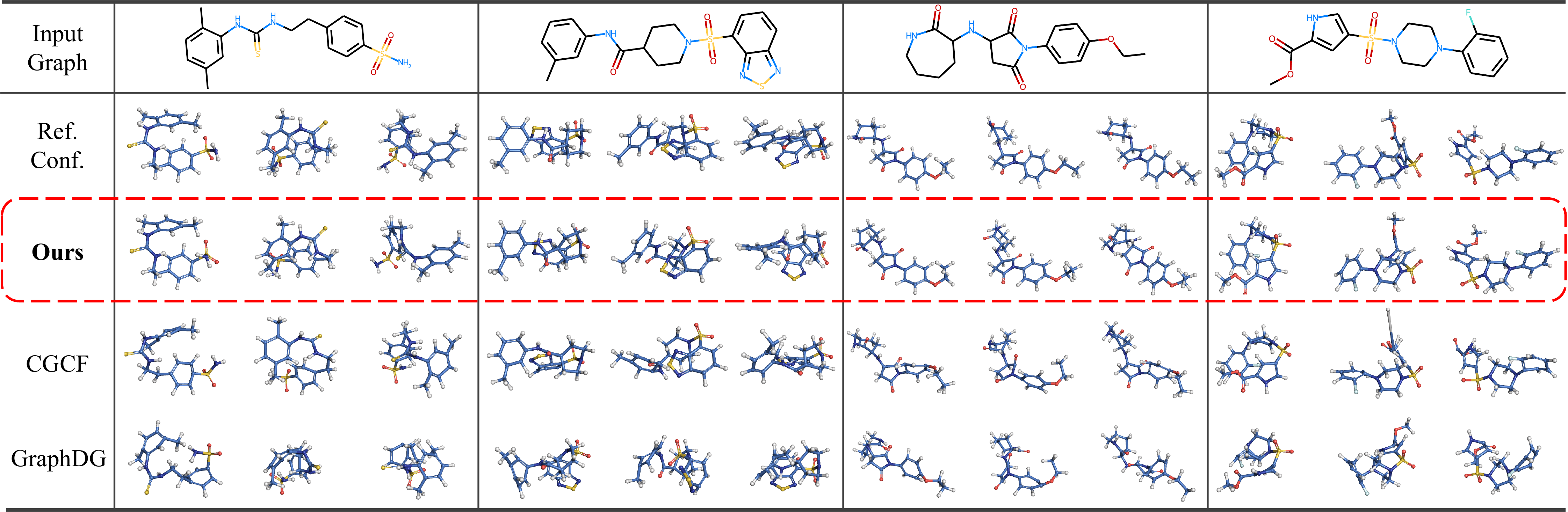}
    \caption{Visualization of conformations generated by different approaches. For each method, we sample multiple conformations and pick ones
    that are best-aligned with the reference ones, based on four random molecular graphs from the test set of GEOM-Drugs.}
    \label{fig:exp_drugs}
    \vspace{-10pt}    
\end{figure*}

Following existing works on conformation generation, we evaluate the proposed method using the following three standard tasks:
\begin{itemize}
\item\textbf{Conformation Generation} tests models' capacity to learn distribution of conformations by measuring the quality of generated conformations compared with reference ones (Section~\ref{subsec:conformation_generation}).
\item\textbf{Distributions Over Distances} evaluates the discrepancy with respect to distance geometry between the generated conformations and the ground-truth conformations (Section~\ref{subsec: distribution_over_distance}). 
\item\textbf{Property Prediction} is first proposed in~\citet{simm2020GraphDG}, which estimates chemical properties for molecular graphs based on a set of sampled conformations (Section~\ref{subsec: molecular_property}). This can be useful in a variety of applications such as drug discovery, where we want to predict molecular properties from microscopic structures.
\end{itemize}

Below we describe setups that are shared across the tasks. Additional setups are provided in the task-specific sections.

\textbf{Data}
Following~\citet{xu2021cgcf}, we use the {GEOM-QM9} and {GEOM-Drugs}~\cite{axelrod2020geom} datasets for the conformation generation task. We randomly draw 40,000 molecules and select the 5 most likely conformations\footnote{sorted by energy} for each molecule,
and draw 200 molecules from the remaining data, resulting in a training set with 200,000 conformations and a test set with 22,408 and 14,324 conformations for GEOM-QM9 and GEOM-Drugs respectively. 
We evaluate the distance modeling task on the {ISO17} dataset~\cite{simm2020GraphDG}.
We use the default split of~\citet{simm2020GraphDG}, resulting in a training set with 167 molecules and the test set with 30 molecules.
For property prediction task, we randomly draw 30 molecules from the GEOM-QM9 dataset with the training molecules excluded.


\textbf{Baselines}
We compare our approach with four state-of-the-art methods for conformation generation. In specific, \textbf{CVGAE}~\cite{mansimov19molecular} is built upon conditional variational graph autoencoder, which directly generates atomic coordinates based on molecular graphs. \textbf{GraphDG}~\cite{simm2020GraphDG} and \textbf{CGCF}~\cite{xu2021cgcf} are distance-based methods, which rely on post-processing algorithms to convert distances to final conformations. The difference of the two methods lies in how the neural networks are designed, i.e., VAE vs. Continuous Flow. \textbf{RDKit}~\cite{sereina2015rdkit} is a classical Euclidean Distance Geometry-based approach.
Results of all baselines across experiments are obtained by running their official codes unless stated.

\textbf{Model Configuration}
\method is implemented in Pytorch~\cite{pytorch2017automatic}.
The GINs is implemented with $N=4$ layers and the hidden dimension is set as 256 across all modules. For training, we use an exponentially-decayed learning rate starting from 0.001 with a decay rate of 0.95. The model is optimized with Adam~\cite{kingma2014adam} optimizer on a single Tesla V100 GPU. All hyperparameters related to noise levels as well as annealed Langevin dynamics are selected according to~\citet{song2020scoretech}. See Supplementary material for full details.

\subsection{Conformation Generation}
\label{subsec:conformation_generation}
\textbf{Setup}
The first task is to generate conformations with high \textit{diversity} and \textit{accuracy} based on molecular graphs. For each molecular graph in the test set, we sample twice as many conformations as the reference ones from each model. 
To measure the discrepancy between two conformations, following existing work~\cite{hawkins2017conformation,li2021conformationguided}, we adopt the Root of Mean Squared Deviations (RMSD) of the heavy atoms between aligned conformations:
\begin{equation}
\label{eq: rmsd}
\operatorname{RMSD}(\mathbf{R}, \mathbf{\hat{R}}) = \min\limits_{\Phi} \Big(
\frac{1}{n}\sum\limits_{i=1}^n \Vert \Phi(\mathbf{R}_i) - \hat{\mathbf{R}}_i \Vert^2 
\Big)^{\frac{1}{2}},
\end{equation}
where $n$ is the number of heavy atoms and $\Phi$ is an alignment function that aligns two conformations by rotation and translation. Following~\citet{xu2021cgcf}, we report the \textbf{Cov}erage (COV) and \textbf{Mat}ching (MAT) score to measure the diversity and accuracy respectively:
\begin{equation}
\begin{aligned}
\operatorname{COV}(S_g, S_r) &= \frac{ \Big\vert
\Big\{\mathbf{R}\in S_r \vert \operatorname{RMSD}(\mathbf{R}, \hat{\mathbf{R}})< \delta,  \hat{\mathbf{R}} \in S_g\Big\}
\Big\vert} {\vert S_r \vert} \\
\operatorname{MAT}(S_g, S_r) &= \frac{1}{\vert S_r \vert}
\sum\limits_{\mathbf{R} \in S_r}
\min\limits_{\hat{\mathbf{R}} \in S_g} \operatorname{RMSD}(\mathbf{R}, \hat{\mathbf{R}}),
\end{aligned}
\end{equation}
where $S_g$ and $S_r$ are generated and reference conformation set respectively. $\delta$ is a given RMSD threshold. In general, a higher COV score indicates a better diversity, while a lower MAT score indicates a better accuracy.

\textbf{Results}
We evaluate the mean and median COV and MAT scores on both GEOM-QM9 and GEOM-Drugs datasets for all baselines, where mean and median are taken over all molecules in the test set.
Results in Table~\ref{tab:confgen} show that our \method achieves the state-of-the-art performance on all four metrics. In particular, as a score-based model that directly estimates the gradients of the log density of atomic coordinates, our approach consistently outperforms other neural models by a large margin. By contrast, the performance of GraphDG and CGCF is inferior to ours due to their two-stage generation procedure, which incurs extra errors when converting distances to conformations.
The CVGAE seems to struggle with both of the datasets as it fails to model the roto-translation equivariance of 
conformations.

When compared with the rule-based method, our approach outperforms the RDKit on 7 out of 8 metrics without relying on any post-processing algorithms, e.g., force fields, as done in previous work~\cite{mansimov19molecular,xu2021cgcf}, making it appealing in real-world applications.
Figure~\ref{fig:exp_drugs} visualizes several conformations that are best-aligned with the reference ones generated by different methods, which demonstrates \method{'}s strong ability to generate realistic and diverse conformations.

\begin{table}[htbp]
    \centering
    \vspace{-7pt}
    
    \caption{COV and MAT scores of different approaches on GEOM-QM9 and GEOM-Drugs datasets.
    The threshold $\delta$ is set as $0.5$\AA~for QM9 and $1.25$\AA~for Drugs following~\citet{xu2021cgcf}. See Supplementary material for additional results.
    }
    \label{tab:confgen}
    \resizebox{\columnwidth}{!}{
    \scalebox{1}{
        \begin{tabular}{c|l | cccc}
\toprule
  \multirow{2}{*}{Dataset} 
& \multirow{2}{*}{Method}
& \multicolumn{2}{c}{COV (\%)} 
& \multicolumn{2}{c}{MAT (\AA)} \\

& & Mean & Median & Mean & Median \\
\midrule
\multirow{5}{*}{QM9}
& RDKit & 83.26 & 90.78 & 0.3447 & 0.2935 \\
& CVGAE & 0.09 & 0.00 & 1.6713 & 1.6088 \\
& GraphDG & 73.33 & 84.21 & 0.4245 & 0.3973 \\
& CGCF & 78.05 & 82.48 & 0.4219 & 0.3900 \\
\cmidrule{2-6}
& \textbf{\method} & \textbf{88.49} & \textbf{94.13} & \textbf{0.2673} & \textbf{0.2685} \\

\midrule
\multirow{5}{*}{Drugs}
& RDKit & 60.91 & 65.70 & 1.2026 & \textbf{1.1252} \\
& CVGAE & 0.00 & 0.00 & 3.0702 & 2.9937 \\
& GraphDG & 8.27 & 0.00 & 1.9722 & 1.9845 \\
& CGCF & 53.96 & 57.06 & 1.2487 & 1.2247 \\
\cmidrule{2-6}
& \textbf{\method} & \textbf{62.15} & \textbf{70.93} & \textbf{1.1629} & 1.1596 \\

\bottomrule
\end{tabular}
    }}
    \vspace{-7pt}
\end{table}

\begin{table}[thbp]
    \centering
    \vspace{-7pt}    
    \caption{Accuracy of the distributions over distances generated by different approaches compared to the ground-truth. Two different metrics are used: mean and median MMD between ground-truth conformations and generated ones. Results of baselines are taken from~\citet{xu2021cgcf}.
    }
    \label{tab:distgen}
    \resizebox{\columnwidth}{!}{
    \scalebox{1}{
        \begin{tabular}{l | cc  cc  cc}
\toprule
\multirow{2}{*}{Method} & 
\multicolumn{2}{c}{Single} &
\multicolumn{2}{c}{Pair} &
\multicolumn{2}{c}{All} \\
& Mean & Median & Mean & Median & Mean & Median \\
\midrule
RDKit & 3.4513 & 3.1602 & 3.8452 & 3.6287 & 4.0866 & 3.7519 \\
CVGAE & 4.1789 & 4.1762 & 4.9184 & 5.1856 & 5.9747 & 5.9928 \\
GraphDG & 0.7645 & 0.2346 & 0.8920 & 0.3287 & 1.1949 & 0.5485 \\
CGCF & 0.4490 & \textbf{0.1786} & 0.5509 & \textbf{0.2734} & 0.8703 & 0.4447\\
\midrule
\textbf{\method} & \textbf{0.3684} & 0.2358 & \textbf{0.4582} & 0.3206 & \textbf{0.6091} & \textbf{0.4240} \\
\bottomrule
\end{tabular}

    }}
    \vspace{-6pt}    
\end{table}


\subsection{Distributions Over Distances}
\label{subsec: distribution_over_distance}

\textbf{Setup}
The second task is to match the generated distributions over distances with the ground-truth.
As \method is not designed for modeling molecular distance geometry, we treat interatomic distances calculated from generated conformations as generated distance distributions.
For each molecular graph in the test set, we sample 1000 conformations across all methods.
Following previous work~\cite{simm2020GraphDG, xu2021cgcf}, we report the discrepancy between generated distributions and the ground-truth distributions, by calculating maximum mean discrepancy (MMD)~\cite{gretton2012kernel} using a Gaussian kernel. In specific, for each graph G in the test set, we ignore edges connected with the H atoms and evaluate distributions of all distances $p(\mathbf{d} \mid G)$ (All), pairwise distances $p(d_{ij}, d_{uv} \mid G)$ (Pair), and individual distances $p(d_{ij} \mid G)$ (Single). 


\textbf{Results}
As shown in Table~\ref{tab:distgen}, the CVGAE suffers the worst performance due to its poor capacity.
Despite the superior performance of RDKit on the conformation generation task, its performance falls far short of other neural models. This is because RDKit incorporates an additional molecular force field~\cite{rappe1992uff} to find equilibrium states of conformations after the initial coordinates are generated. 
By contrast, although not tailored for molecular distance geometry, our \method still achieves competitive results on all evaluation metrics. 
In particular, the \method outperforms the state-of-the-art method CGCF on 4 out of 6 metrics, and achieves comparable results on the other two metrics.




\subsection{Property Prediction}
\label{subsec: molecular_property}

\textbf{Setup}
The third task is to predict the \textit{ensemble property} \cite{axelrod2020geom} of molecular graphs.
Specifically, the ensemble property of a molecular graph is calculated by aggregating the property of different conformations.
For each molecular graph in the test set, we first calculate the energy, HOMO and LUMO of all the ground-truth conformations using the quantum chemical calculation package Psi4~\cite{Smith2020Psi41O}. 
Then, we calculate ensemble properties including average energy $ \overline{E} $, lowest energy $E_\text{min}$, average HOMO-LUMO gap $ \overline{\Delta \epsilon} $, minimum gap $\Delta \epsilon_\text{min}$, and maximum gap $\Delta \epsilon_\text{max}$ based on the conformational properties for each molecular graph.
Next, we use \method and other baselines to generate 50 conformations for each molecular graph and calculate the above-mentioned ensemble properties in the same way.
We use mean absolute error (MAE) to measure the accuracy of ensemble property prediction.
We exclude CVGAE from this analysis due to its poor quality following~\citet{simm2020GraphDG}.

\textbf{Results}
As shown in Table \ref{tab:property_pred}, \method outperforms all the machine learning-based methods by a large margin and achieves better accuracy than RDKit in lowest energy $E_\text{min}$ and maximum HOMO-LUMO gap $\Delta\epsilon_\text{max}$.
Closer inspection on the generated samples shows that, although \method has better generation quality in terms of diversity and accuracy, it sometimes generates a few outliers which have negative impact on the prediction accuracy. 
We further present the median of absolute errors in Table \ref{tab:property_pred_median}. When measured by median absolute error, \method has the best accuracy in average energy and the accuracy of other properties is also improved, which reveals the impact of outliers. 

\begin{table}[htbp]
    \centering
    \vspace{-7pt}    
    \caption{Mean absolute errors (MAE) of predicted ensemble properties. Unit: eV.}
    \label{tab:property_pred}
    \resizebox{\columnwidth}{!}{
    \scalebox{1}{
        \begin{tabular}{l | ccccc}
\toprule
Method & $ \overline{E} $ & $E_\text{min}$ & $ \overline{\Delta\epsilon} $ & $\Delta \epsilon_\text{min}$  & $\Delta \epsilon_\text{max}$ \\
\midrule
RDKit & \bf 0.9233 & 0.6585 & \bf 0.3698 & \bf 0.8021 & 0.2359 \\
GraphDG & 9.1027 & 0.8882 & 1.7973 & 4.1743 & 0.4776 \\
CGCF & 28.9661 & 2.8410 & 2.8356 & 10.6361 & 0.5954 \\
\bf \method & 2.7886 & \bf 0.1765 & 0.4688 & 2.1843 & \bf 0.1433 \\
\bottomrule
\end{tabular}

    }}
    \vspace{-4pt}    
\end{table}

\begin{table}[htbp]
    \centering
    \vspace{-7pt}    
    \caption{Median of absolute prediction errors. Unit: eV.}
    \label{tab:property_pred_median}
    \resizebox{\columnwidth}{!}{
    \scalebox{1}{
        \begin{tabular}{l | ccccc}
\toprule
Method & $ \overline{E} $ & $E_\text{min}$ & $ \overline{\Delta\epsilon} $ & $\Delta \epsilon_\text{min}$  & $\Delta \epsilon_\text{max}$ \\
\midrule
RDKit & 0.8914 & 0.6629 & \bf 0.2947 & \bf 0.5196 & 0.1617 \\
\bf \method & \bf 0.5328 & \bf 0.1145 & 0.3207 & 0.7365 & \bf 0.1337 \\
\bottomrule
\end{tabular}

    }}
    \vspace{-4pt}    
\end{table}

\begin{table}[htbp]
    \centering
    \vspace{-7pt}    
    \caption{Ablation study on the performance of \method. The results of CGCF are taken directly from Table~\ref{tab:confgen}. The threshold $\delta$ is set as $0.5$\AA~for QM9 and $1.25$\AA~for Drugs following Table~\ref{tab:confgen}.}
    \label{tab:ablation}
    \resizebox{\columnwidth}{!}{
    \scalebox{1}{
        \begin{tabular}{c | l | cccc}
\toprule
\multirow{2}{*}{Dataset} 
& \multirow{2}{*}{Method} 
& \multicolumn{2}{c}{COV (\%)} 
& \multicolumn{2}{c}{MAT (\AA)} \\
& & Mean & Median & Mean & Median \\
\midrule
\multirow{3}{*}{QM9} 
& CGCF & 78.05 & 82.48 & 0.4219 & 0.3900 \\
& \methodd & 81.94 & 85.80 & 0.3861 & 0.3571 \\
& \method & \textbf{88.49} & \textbf{94.13} & \textbf{0.2673} & \textbf{0.2685} \\   
\midrule
\multirow{3}{*}{Drugs} 
& CGCF & 53.96 & 57.06 & 1.2487 & 1.2247 \\
& \methodd & 53.40 & 52.58 & 1.2493 & 1.2366 \\
& \method & \textbf{62.15} & \textbf{70.93} & \textbf{1.1629} & \textbf{1.1596} \\    
\bottomrule
\end{tabular}

    }}
    \vspace{-4pt}    
\end{table}

\subsection{Ablation Study}
\label{subsec:ablation}
So far, we have justified the effectiveness of the proposed \method on multiple tasks. 
However, it remains unclear whether the proposed strategy, i.e., directly estimating the gradient fields of log density of atomic coordinates, contributes to the superior performance of \method.
To gain insights into the working behaviours of \method, we conduct the ablation study in this section.

We adopt a variant of \method, called \methodd, which works in a way similar to distance-based methods~\cite{simm2020GraphDG, xu2021cgcf}.
In specific, \methodd bypasses chain rule and directly generates interatomic distances via annealed Langevin dynamics using estimated gradient fields of interatomic distances, i.e., $\mathbf{s}_\theta(\mathbf{d}, \sigma)$ (Section~\ref{subsec:scorenet}), after which we convert them to conformations by the same post-processing algorithm as used in~\citet{xu2021cgcf}. 
Following the setup of Section~\ref{subsec:conformation_generation}, we evaluate two approaches on the conformation generation task and summarize the results in Table~\ref{tab:ablation}. 
As presented in Table~\ref{tab:ablation}, we observe that \methodd performs significantly worse than \method on both datasets, and performs on par with the state-of-the-art method CGCF. The results indicate that
the two-stage generation procedure of distance-based methods does harm the quality of conformations, and our strategy effectively addresses this issue, leading to much more accurate and diverse conformations.

\section{Conclusion and Future Work}
\label{sec:conclusion}



In this paper, we propose a novel principle for molecular conformation generation called \method, where samples are generated using Langevin dynamics within one stage with physically inspired gradient fields of log density of atomic coordinates. 
We novelly develop an algorithm to effectively estimate these gradients and meanwhile preserve their roto-translation equivariance.
Comprehensive experiments over multiple tasks verify that \method outperforms previous state-of-the-art baselines by a significant margin.
Our future work will include extending our \method approach to 3D molecular design tasks and many-body particle systems.

\section*{Acknowledgments}

We would like to thank Yiqing Jin for refining the figures used in this paper.
This project is supported by the Natural Sciences and Engineering Research Council (NSERC) Discovery Grant, the Canada CIFAR AI Chair Program, collaboration grants between Microsoft Research and Mila, Samsung Electronics Co., Ldt., Amazon Faculty Research Award, Tencent AI Lab Rhino-Bird Gift Fund and a NRC Collaborative R\&D Project (AI4D-CORE-06). This project was also partially funded by IVADO Fundamental Research Project grant PRF-2019-3583139727.

\bibliography{reference}
\bibliographystyle{icml2021}

\clearpage
\newpage
\onecolumn
\appendix
{\Large \textbf{Supplementary Material}}
\section{Proof of Proposition~\ref{prop:roto}}
Let $\mathbf{R}= (\mathbf{r}_1, \mathbf{r}_2, \cdots, \mathbf{r}_{\vert V \vert}) \in \mathbb{R}^{\vert V \vert \times 3}$ denote a molecular conformation. 
Let $\gamma_{\mathbf{t}}: \mathbb{R}^{\vert V \vert \times 3} \rightarrow \mathbb{R}^{\vert V \vert \times 3}$ denote any 3D translation function where $\gamma_{\mathbf{t}}(\mathbf{R})_i \coloneqq \mathbf{r}_i + \mathbf{t}$, and let $\rho_D: \mathbb{R}^{\vert V \vert \times 3} \rightarrow \mathbb{R}^{\vert V \vert \times 3}$ denote any 3D rotation function whose rotation matrix representation is $D \in \mathbb{R}^{3 \times 3}$, i.e., $\rho_D(\mathbf{R})_i = D\mathbf{r_i} \in \mathbb{R}^3$.
In our context, we say a score function of atomic coordinates $\mathbf{s}: \mathbb{R}^{\vert V \vert \times 3} \rightarrow \mathbb{R}^{\vert V \vert \times 3}$ is roto-translation equivariant if it satisfies:

\begin{equation}
\label{eq:suppl_roto}
    \mathbf{s} \circ \gamma_{\mathbf{t}} \circ \rho_D (\mathbf{R}) = \rho_D \circ \mathbf{s}(\mathbf{R}),
\end{equation}
Intuitively, the above equation says that applying any translation $\gamma_{\mathbf{t}}$ to the input has no effect on the output, and applying any rotation $\rho_D$ to the input has the same effect as applying it to the output, i.e., the gradients rotate together with the molecule system and are invariant under translation.
\begin{proof}
Denote $\hat{\mathbf{R}} = \gamma_{\mathbf{t}} \circ \rho_D (\mathbf{R})$. According to the definition of translation and rotation function, we have $\hat{\mathbf{R}}_i = \hat{\mathbf{r}}_i = D\mathbf{r}_i + \mathbf{t}$.
According to Eq.~\ref{eq: chainrule}, we have
\begin{equation}
\label{eq: proof}
\begin{aligned}
    \forall i, \big (\mathbf{s} \circ \gamma_{\mathbf{t}} \circ \rho_D (\mathbf{R})\big )_i & = 
    \mathbf{s}(\hat{\mathbf{R}})_i \\
    &= \sum\limits_{j \in N(i)}
    \frac{1}{\hat{d}_{ij}} \cdot \mathbf{s}(\hat{\mathbf{d}})_{ij} \cdot (\hat{\mathbf{r}}_i - \hat{\mathbf{r}}_j) \\
    &= \sum\limits_{j \in N(i)}
    \frac{1}{{d}_{ij}} \cdot \mathbf{s}({\mathbf{d}})_{ij} \cdot \big((D\mathbf{r}_i + \mathbf{t}) - (D\mathbf{r}_j + \mathbf{t})\big)
    \\
    &= \sum\limits_{j \in N(i)}
    \frac{1}{{d}_{ij}} \cdot \mathbf{s}({\mathbf{d}})_{ij} \cdot D(\mathbf{r}_i - \mathbf{r}_j) \\
    &=D \Big (\sum\limits_{j \in N(i)}
    \frac{1}{{d}_{ij}} \cdot \mathbf{s}({\mathbf{d}})_{ij} \cdot (\mathbf{r}_i - \mathbf{r}_j) \Big )\\
    &= D \mathbf{s}({\mathbf{R}})_i \\
    &= \big(\rho_D \circ \mathbf{s}(\mathbf{R})\big)_i.
\end{aligned}
\end{equation}
Here $d_{ij} = \hat{d}_{ij}$ and $\mathbf{d} = \hat{\mathbf{d}}$ because rotation and translation will not change interatomic distances. 
Together with Eq.~\ref{eq:suppl_roto} and Eq.~\ref{eq: proof}, we conclude that the score network defined in Eq.~\ref{eq: chainrule} is roto-translation equivariant.
\end{proof}

\section{Additional Hyperparameters}
We summarize additional hyperparameters of our \method in Table~\ref{tab:suppl_hyperparameters}, including the biggest noise level $\sigma_1$, the smallest noise level $\sigma_L$, the number of noise levels $L$, the number of sampling steps per noise level $T$, the smallest step size $\epsilon$, the batch size and the training epochs.
\begin{table}[htbp]
    \centering
    \vspace{-7pt}    
    \caption{Additional hyperparameters of our \method.}
    \label{tab:suppl_hyperparameters}
        \begin{tabular}{cccccccc}
    \toprule
    Dataset & $\sigma_1$ & $\sigma_L$ & $L$ & $T$ & $\epsilon$ & Batch size & Training epochs\\ 
    \midrule
    GEOM-QM9 & 10 & 0.01 & 50 &  100 & 2.4e-6 & 128 & 200\\
    GEOM-Drugs & 10 & 0.01 & 50 &  100 & 2.4e-6 & 128 & 200\\
    ISO17 & 3 & 0.1 & 30 &  100 & 2.0e-4 & 128 & 100\\    
    \bottomrule
\end{tabular}
\end{table}

\clearpage
\section{Additional Experiments}

We present more results of the \textbf{Cov}erage (COV) score at different threshold $\delta$ for GEOM-QM9 and GEOM-Drugs datasets in Tables~\ref{tab:suppl_exp_qm9} and~\ref{tab:suppl_exp_drugs} respectively.
Results in Tables~\ref{tab:suppl_exp_qm9} and~\ref{tab:suppl_exp_drugs} indicate that the proposed \method consistently outperforms previous state-of-the-art baselines, including GraphDG and CGCF.
In addition, we also report the \textbf{Mis}match (MIS) score defined as follows:
\begin{equation}
\operatorname{MIS}(S_g, S_r) = \frac{1}{\vert S_g \vert} \Big\vert
\Big\{\mathbf{R}\in S_g \vert \operatorname{RMSD}(\mathbf{R}, \hat{\mathbf{R}}) > \delta,  \forall \hat{\mathbf{R}} \in S_r\Big\}
\Big\vert,
\end{equation}
where $\delta$ is the threshold. 
The MIS score measures the fraction of generated conformations that are not matched by any ground-truth conformation in the reference set given a threshold $\delta$. 
A lower MIS score indicates less bad samples and better generation quality.
As shown in Tables~\ref{tab:suppl_exp_qm9} and~\ref{tab:suppl_exp_drugs}, the MIS scores of \method are consistently lower than the competing baselines, which demonstrates that our \method generates less invalid conformations compared with the existing models.

\begin{table}[htbp]
    \centering
    \vspace{-7pt}    
    \caption{COV and MIS scores of different approaches on GEOM-QM9 dataset at different threshold $\delta$.}
    \label{tab:suppl_exp_qm9}
    \resizebox{\columnwidth}{!}{
    \scalebox{1}{
        \newcommand{\graycell}{\cellcolor{gray!25}}

\begin{tabular}{c | ccc | ccc | ccc | ccc}
\toprule
QM9 & \multicolumn{3}{c|}{Mean COV (\%)} & \multicolumn{3}{c|}{Median COV (\%)} & \multicolumn{3}{c|}{Mean MIS (\%)} & \multicolumn{3}{c}{Median MIS (\%)} \\
\midrule
$\delta$ (\AA) & GraphDG & CGCF & \textbf{\method} & GraphDG & CGCF & \textbf{\method} & GraphDG & CGCF & \textbf{\method} & GraphDG & CGCF & \textbf{\method} \\
\midrule
0.10 & 1.03 & 0.14 & \graycell 17.59 & 0.00 & 0.00 & \graycell 10.29 & 99.70 & 99.93 & \graycell 92.30 & 100.00 & 100.00 & \graycell 96.23 \\
0.20 & 13.05 & 10.97 & \graycell 43.60 & 2.99 & 3.95 & \graycell 37.92 & 96.12 & 96.91 & \graycell 81.67 & 99.04 & 99.04 & \graycell 85.65 \\
0.30 & 32.26 & 31.02 & \graycell 61.94 & 18.81 & 22.94 & \graycell 59.66 & 87.15 & 89.21 & \graycell 73.06 & 94.44 & 94.44 & \graycell 77.27 \\
0.40 & 53.53 & 53.65 & \graycell 75.45 & 50.00 & 52.63 & \graycell 80.64 & 72.60 & 78.35 & \graycell 65.38 & 82.63 & 82.63 & \graycell 70.00 \\
0.50 & 73.33 & 78.05 & \graycell 88.49 & 84.21 & 82.48 & \graycell 94.13 & 56.09 & 63.51 & \graycell 53.56 & 64.66 & 64.66 & \graycell 56.59 \\
\midrule
0.60 & 88.24 & 94.85 & \graycell 97.71 & 98.83 & 98.79 & \graycell 100.00 & 40.36 & 44.82 & \graycell 34.78 & 43.73 & 43.73 & \graycell 35.86 \\
0.70 & 95.93 & 99.05 & \graycell 99.52 & 100.00 & 100.00 & \graycell 100.00 & 27.93 & 29.64 & \graycell 21.00 & 23.38 & 23.38 & \graycell 15.64 \\
0.80 & 98.70 & 99.47 & \graycell 99.68 & 100.00 & 100.00 & \graycell 100.00 & 19.15 & 20.98 & \graycell 12.86 & 10.72 & 10.72 & \graycell 5.23 \\
0.90 & 99.33 & 99.50 & \graycell 99.77 & 100.00 & 100.00 & \graycell 100.00 & 12.76 & 16.74 & \graycell 8.98 & 3.65 & 3.65 & \graycell 1.53 \\
1.00 & 99.48 & 99.50 & \graycell 99.86 & 100.00 & 100.00 & \graycell 100.00 & 8.00 & 14.19 & \graycell 6.76 & 0.47 & 0.47 & \graycell 0.36 \\
\midrule
1.10 & 99.51 & 99.51 & \graycell 99.91 & 100.00 & 100.00 & \graycell 100.00 & 4.99 & 12.26 & \graycell 5.57 & 0.00 & 0.00 & \graycell 0.00 \\
1.20 & 99.51 & 99.51 &\graycell 99.94 & 100.00 & 100.00 & \graycell 100.00 & 2.95 & 9.68 & \graycell 3.48 & 0.00 & 0.00 & \graycell 0.00 \\
1.30 & 99.51 & 99.51 & \graycell 99.96 & 100.00 & 100.00 & \graycell 100.00 & 1.65 & 7.48 & \graycell 1.91 & 0.00 & 0.00 & \graycell 0.00 \\
1.40 & 99.51 & 99.51 & \graycell 99.96 & 100.00 & 100.00 & \graycell 100.00 & 0.84 & 5.94 & \graycell 1.05 & 0.00 & 0.00 & \graycell 0.00 \\
1.50 & 99.52 & 99.51 & \graycell 99.97 & 100.00 & 100.00 & \graycell 100.00 & 0.41 & 4.94 & \graycell 0.70 & 0.00 & 0.00 & \graycell 0.00 \\
\bottomrule

\end{tabular}

    }}
\end{table}

\begin{table}[htbp]
    \centering
    \vspace{-7pt}    
    \caption{COV and MIS scores of different approaches on GEOM-Drugs dataset at different threshold $\delta$.}
    \label{tab:suppl_exp_drugs}
    \resizebox{\columnwidth}{!}{
    \scalebox{1}{
        \newcommand{\graycell}{\cellcolor{gray!25}}

\begin{tabular}{c | ccc | ccc | ccc | ccc}
\toprule
Drugs & \multicolumn{3}{c|}{Mean COV (\%)} & \multicolumn{3}{c|}{Median COV (\%)} & \multicolumn{3}{c|}{Mean MIS  (\%)} & \multicolumn{3}{c}{Median MIS (\%)} \\
\midrule
$\delta$ (\AA) & GraphDG & CGCF & \textbf{\method} & GraphDG & CGCF & \textbf{\method} & GraphDG & CGCF & 
\textbf{\method} & GraphDG & CGCF & \textbf{\method} \\
\midrule

0.25 & 0.00 & 0.06 & \graycell 0.17 & 0.00 & 0.00 & \graycell 0.00 & 100.00 & 99.99 & \graycell 99.97 & 100.00 & 100.00 & \graycell 100.00 \\
0.50 & 0.26 & 0.80 & \graycell 1.15 & 0.00 & 0.00 & \graycell 0.00 & 99.95 & 99.80 & \graycell 99.52 & 100.00 & 100.00 & \graycell 100.00 \\
0.75 & 0.75 & 5.81 & \graycell 9.15 & 0.00 & 0.00 & \graycell 0.50 & 99.69 & 97.86 & \graycell 96.94 & 100.00 & 100.00 & \graycell 99.75 \\
1.00 & 2.39 & 24.67 & \graycell 30.60 & 0.00 & 11.81 & \graycell 18.89 & 99.14 & 90.82 & \graycell 89.63 & 100.00 & 96.50 & \graycell 95.58 \\
\midrule
1.25 & 8.27 & 53.96 & \graycell 62.15 & 0.00 & 57.06 & \graycell 70.93 & 97.92 & 78.32 & \graycell 76.58 & 100.00 & 86.28 & \graycell 84.48 \\
1.50 & 19.96 & 79.37 & \graycell 86.62 & 4.00 & 92.46 & \graycell 98.79 & 94.40 & 63.80 & \graycell 60.06 & 99.14 & 66.39 & \graycell 63.81 \\
1.75 & 36.86 & 91.47 & \graycell 96.53 & 26.58 & 100.00 & \graycell 100.00 & 87.68 & 49.72 & \graycell 43.63 & 95.83 & 47.09 & \graycell 41.72 \\
2.00 & 55.79 & 96.73 & \graycell 98.62 & 55.26 & 100.00 & \graycell 100.00 & 76.99 & 37.53 & \graycell 29.80 & 87.35 & 30.90 & \graycell 22.44 \\
\midrule
2.25 & 71.43 & 99.05 & \graycell 99.83 & 80.00 & 100.00 & \graycell 100.00 & 61.76 & 27.30 & \graycell 18.68 & 69.74 & 20.07 & \graycell 10.93 \\
2.50 & 83.53 & 99.47 & \graycell 100.00 & 95.45 & 100.00 & \graycell 100.00 & 44.32 & 18.97 & \graycell 11.09 & 42.96 & 12.33 & \graycell 3.31 \\
2.75 & 91.09 & 99.60 & \graycell 100.00 & 100.00 & 100.00 & \graycell 100.00 & 27.92 & 12.52 & \graycell 6.32 & 16.67 & 6.82 & \graycell 0.74 \\
3.00 & 95.00 & 99.96 & \graycell 100.00 & 100.00 & 100.00 & \graycell 100.00 & 15.97 & 7.67 & \graycell 3.36 & 2.46 & 3.32 & \graycell 0.00 \\

\bottomrule

\end{tabular}

    }}
\end{table}

\clearpage
\section{More Generated Samples}
We present more visualizations of generated conformations from our \method in Figure~\ref{fig:exp_suppl}, including models trained on GEOM-QM9 and GEOM-drugs datasets.
\begin{figure*}[htbp]
	\centering
    \includegraphics[width=\linewidth]{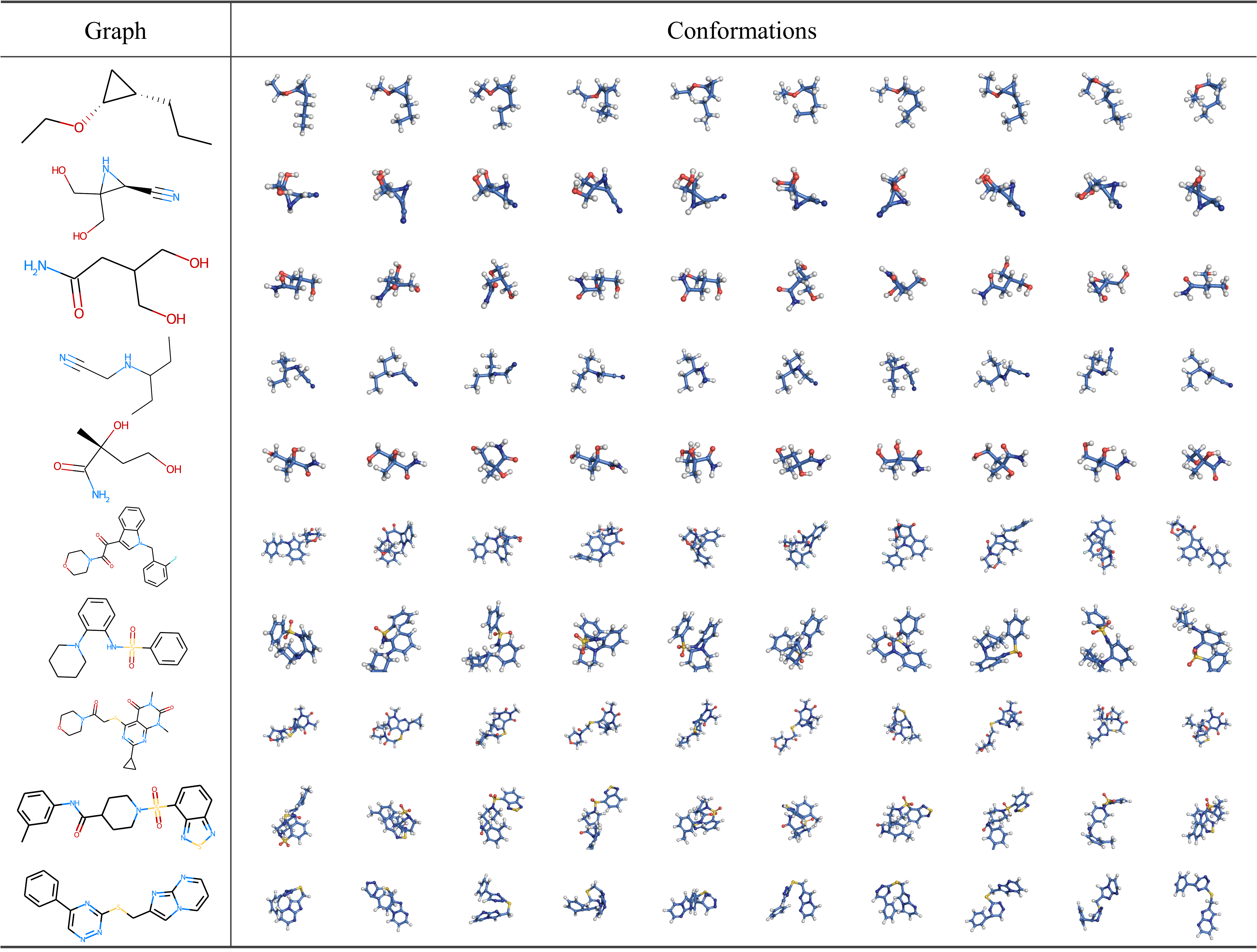}
    \caption{Visualization of conformations generated by our \method. For each molecular graph, we randomly sample 10 generated conformations from our model. For the top 5 rows, the molecular graphs are drawn from the GEOM-QM9 test dataset, while for the bottom 5 rows, the molecular graphs are drawn from the GEOM-drugs test dataset.}
    \label{fig:exp_suppl}
\end{figure*}

\end{document}